\documentclass[runningheads]{llncs}

\usepackage[T1]{fontenc}

\usepackage{amssymb}
\usepackage{amsfonts}
\usepackage{amsmath}
\usepackage{graphicx}
\usepackage{microtype}

\usepackage{color}
\usepackage{url}
\usepackage{hyperref}

\usepackage{orcidlink}
\usepackage{enumitem}
\usepackage{cite}

\usepackage{caption}
\usepackage{subcaption}
\usepackage{booktabs}
\usepackage{multirow}
\usepackage{siunitx}
\usepackage{tabularx}
\usepackage[table]{xcolor}
\usepackage{array}

\definecolor{lightgray}{gray}{0.95}

\usepackage{lmodern}
\usepackage{etoolbox}
\robustify\bfseries

\usepackage[ruled,linesnumbered,noend]{algorithm2e}
\SetKw{Continue}{continue}
\SetKw{Break}{break}

\providetoggle{beginFlag}
\settoggle{beginFlag}{true}

\SetAlgoSkip{SkipBeforeAndAfter}

\makeatletter
\def\@seccntformat#1{\@ifundefined{#1@cntformat}%
   {\csname the#1\endcsname\quad}
   {\csname #1@cntformat\endcsname}
}
\let\oldappendix\appendix
\renewcommand\appendix{%
    \oldappendix
    \newcommand{\section@cntformat}{\appendixname~\thesection\quad}
}
\makeatother

\AtBeginDocument{%
  \abovedisplayskip=6pt plus 2pt minus 2pt
  \belowdisplayskip=6pt plus 2pt minus 2pt
  \abovedisplayshortskip=3pt plus 1pt minus 1pt
  \belowdisplayshortskip=3pt plus 1pt minus 1pt
}

\makeatletter
\renewcommand\section{\@startsection{section}{1}{\z@}%
                       {-12\p@ \@plus -3\p@ \@minus -2\p@}%
                       {8\p@ \@plus 2\p@ \@minus 2\p@}%
                       {\normalfont\large\bfseries\boldmath
                        \rightskip=\z@ \@plus 8em\pretolerance=10000 }}
\renewcommand\subsection{\@startsection{subsection}{2}{\z@}%
                       {-10\p@ \@plus -3\p@ \@minus -2\p@}%
                       {5\p@ \@plus 2\p@ \@minus 2\p@}%
                       {\normalfont\normalsize\bfseries\boldmath
                        \rightskip=\z@ \@plus 8em\pretolerance=10000 }}
\renewcommand\subsubsection{\@startsection{subsubsection}{3}{\z@}%
                       {-10\p@ \@plus -3\p@ \@minus -2\p@}%
                       {-0.5em \@plus -0.22em \@minus -0.1em}%
                       {\normalfont\normalsize\bfseries\boldmath}}
\makeatother

\usepackage{tikz}
\usepackage{xcolor}

\definecolor{solidblue}{HTML}{8cc5e3}
\definecolor{dashedblue}{HTML}{2066a8}
\definecolor{dasheddottedblue}{HTML}{3594cc}

\definecolor{solidred}{HTML}{cc3333}
\definecolor{variationalblue}{HTML}{2066a8}
\definecolor{dottedgray}{HTML}{808080}

\DeclareRobustCommand{\inlinesolidblue}{\tikz[baseline=-0.5ex]{\draw[solidblue, line width=1pt] (0,0) -- (0.4,0);}}
\DeclareRobustCommand{\inlinedashedblue}{\tikz[baseline=-0.5ex]{\draw[dashedblue, line width=1pt, dashed] (0,0) -- (0.4,0);}}
\DeclareRobustCommand{\inlinedasheddottedblue}{\tikz[baseline=-0.5ex]{\draw[dasheddottedblue, line width=1pt, dash dot] (0,0) -- (0.4,0);}}
\DeclareRobustCommand{\inlinesolidred}{\tikz[baseline=-0.5ex]{\draw[solidred, line width=1pt] (0,0) -- (0.4,0);}}
\DeclareRobustCommand{\inlinevariationalblue}{\tikz[baseline=-0.5ex]{\draw[variationalblue, line width=1pt, dashed] (0,0) -- (0.4,0);}}
\DeclareRobustCommand{\inlinedottedgray}{\tikz[baseline=-0.5ex]{\draw[dottedgray, line width=1pt, dotted] (0,0) -- (0.4,0);}}

\newcommand{\M}{\mathcal{M}}
\newcommand{\TqM}{\mathcal{T}_{q}\M}
\newcommand{\TM}{\mathcal{T}\M}
\newcommand{\D}{\mathrm{D}}
\newcommand{\R}{\mathcal{R}}
\newcommand{\qrand}{q_{\text{rand}}}
\newcommand{\qnear}{q_{\text{near}}}

\NewDocumentCommand\T{}{\mathsf{T}}
\NewDocumentCommand\Transpose{m}{ \left.{#1}\right.^\T }
\NewDocumentCommand\LieGroupSE{m}{ \mathrm{SE}(#1) }
\NewDocumentCommand\Inv{m}{{#1}^{-1}}
\NewDocumentCommand\Norm{m}{ \left\Vert#1\right\Vert }
\NewDocumentCommand\ArgMin{m}{ \operatorname*{argmin}_{#1} }
\DeclareMathOperator{\grad}{grad}

\begin{document}

\title{Geometry-Aware Sampling-Based Motion Planning on Riemannian Manifolds}
\titlerunning{Sampling-Based Motion Planning on Riemannian Manifolds}

\author{Phone Thiha Kyaw\orcidlink{0000-0001-8790-8350} \and
Jonathan Kelly\orcidlink{0000-0002-5528-6136}}
\authorrunning{P. T. Kyaw and J. Kelly}
\institute{Institute for Aerospace Studies, University of Toronto, Toronto, Canada\\
\email{\{phone.thiha,jonathan.kelly\}@robotics.utias.utoronto.ca}}
\maketitle

\begin{abstract}
In many robot motion planning problems, task objectives and physical constraints induce non-Euclidean geometry on the configuration space, yet many planners operate using Euclidean distances that ignore this structure.
We address the problem of planning collision-free motions that minimize length under configuration-dependent Riemannian metrics, corresponding to geodesics on the configuration manifold.
Conventional numerical methods for computing such paths do not scale well to high-dimensional systems, while sampling-based planners trade scalability for geometric fidelity.
To bridge this gap, we propose a sampling-based motion planning framework that operates directly on Riemannian manifolds.
We introduce a computationally efficient midpoint-based approximation of the Riemannian geodesic distance and prove that it matches the true Riemannian distance with third-order accuracy.
Building on this approximation, we design a local planner that traces the manifold using first-order retractions guided by Riemannian natural gradients.
Experiments on a two-link planar arm and a 7-DoF Franka manipulator under a kinetic-energy metric, as well as on rigid-body planning in $\LieGroupSE{2}$ with non-holonomic motion constraints, demonstrate that our approach consistently produces lower-cost trajectories than Euclidean-based planners and classical numerical geodesic-solver baselines.
\end{abstract}

\section{Introduction}
\label{sec:introduction}

Robotic motion planning is often posed as a search for collision-free paths through a configuration space.
For many robotic systems, the configuration space is a non-Euclidean manifold.
For example, rigid-body poses live on Lie groups such as $\LieGroupSE{2}$ and $\LieGroupSE{3}$, while articulated manipulators live on products of circles (tori).
More generally, closed-chain or task constraints induce implicit manifolds embedded in a higher-dimensional ambient space~\cite{lavalle2006planning}.
In these settings, planning feasibility and optimality have clear geometric interpretations: feasibility is governed by the intrinsic manifold structure and constraints, while optimality depends on how we measure the cost of motion along the manifold.

A natural way to encode such costs is via a Riemannian metric on the appropriate configuration manifold.
This metric induces a notion of distance on the manifold, and the resulting shortest paths, called geodesics, generalize straight lines in Euclidean spaces to curved manifolds~\cite{lee2018introduction}.
This abstraction recovers classical shortest-path planning in Euclidean configuration spaces as a special case, while also capturing costs that vary smoothly with configuration.
Examples of configuration-dependent metrics include the kinetic-energy metric commonly used in manipulation~\cite{bullo2019geometric,jaquier2022riemannian,li2024riemannian}, as well as group-invariant metrics on Lie groups~\cite{milnor1976curvatures, helgason1979differential}.
In these settings, optimal motion planning becomes the problem of finding collision-free paths that minimize Riemannian arc length under either a constant or smoothly varying metric.

\begin{figure}[!t]
\centering
\vspace{-0.55\baselineskip}
\includegraphics[width=0.89\textwidth]{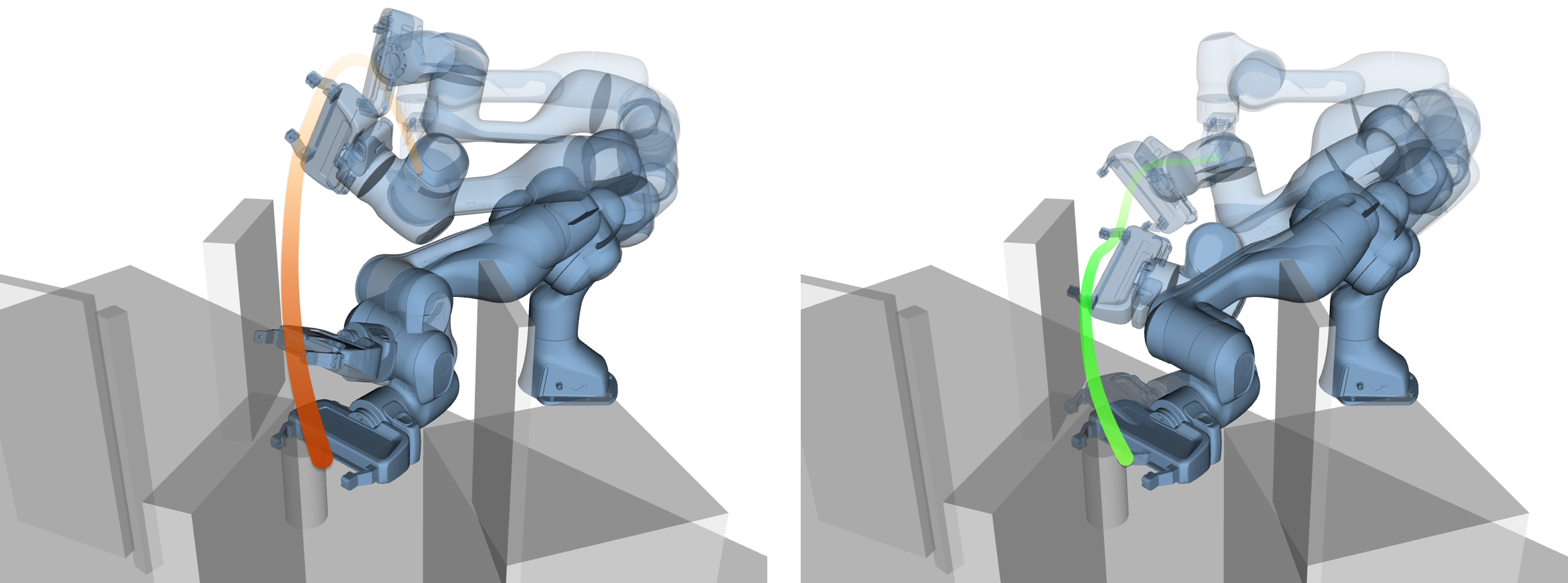}
\vspace{-0.4\baselineskip}
\caption{
Geodesic motion planning for a 7-DoF Franka manipulator in a cluttered environment.
Conventional sampling-based methods compute joint-space shortest paths under the ambient Euclidean metric (red), often ignoring the intrinsic geometry of the configuration manifold (left).
Our geodesic formulation instead recovers minimum-energy trajectories under the kinetic-energy Riemannian metric (green), with path thickness indicating relative energy consumption (right).
}
\vspace{-1.25\baselineskip}
\label{fig:franka}
\end{figure}

\looseness=-1
We study the problem of computing minimum-length geodesics on Riemannian manifolds arising in robotic motion planning.
Finding such paths in high dimensions is computationally expensive.
Classical approaches either directly solve the geodesic ordinary differential equations as a boundary value problem or minimize a variational energy functional~\cite{peyre2010geodesic}.
While effective in low dimensions, these methods scale poorly to high-dimensional robotic systems, struggling to satisfy feasibility and optimality while respecting constraints like joint limits and obstacles in real-time.
Sampling-based motion planning offers a scalable alternative.
Algorithms such as the Rapidly-exploring Random Tree (RRT) and its anytime variants have been widely adopted in robotic motion planning and shown to scale well in high dimensions~\cite{lavalle2001randomized, lavalle2006planning, karaman2011sampling}.
Most existing sampling-based approaches measure distances and interpolate motions using an ambient Euclidean metric.
Consequently, they may ignore the intrinsic geometry of the configuration manifold, leading to motions that violate manifold constraints or are infeasible.
Constrained motion planning techniques combine numerical continuation with sampling-based planning to address these feasibility problems by searching directly on the manifold~\cite{henderson2002multiple, jaillet2012path, kingston2018sampling}.
However, optimality is usually expressed through a fixed Euclidean metric, leaving a gap between manifold-aware feasibility and metric-aware optimality.

\looseness=-1
In this work, we develop geometry-aware subroutines that allow anytime sampling-based planners to optimize Riemannian path length directly on configuration manifolds.
This generalizes prior approaches by supporting both constant and smoothly varying Riemannian metrics (Figure~\ref{fig:franka}).
Our contributions are summarized as follows.
\vspace{-0.5\baselineskip}
\begin{itemize}
\item We propose a midpoint-based approximation of the Riemannian geodesic distance and prove that its approximation error vanishes asymptotically with third-order accuracy.
\item We design a geometry-aware local planner that traces the manifold using
retraction steps along the Riemannian natural gradient under a configuration-dependent metric.
\item We empirically validate our approach on energy-minimizing motion planning problems for serial manipulators and on non-holonomic planning tasks on $\LieGroupSE{2}$, demonstrating consistently lower-cost paths compared to baselines.
\end{itemize}
\section{Related Work}
\label{sec:related_work}

Robot motion generation algorithms often optimize smoothness criteria, such as minimum jerk or acceleration, inspired by human arm movements (see \cite{todorov2004optimality} for a review).
While effective in certain contexts, these models typically neglect the nonlinear coupling and dynamics that are intrinsic to articulated systems like humans and (most) robots.
As a result, linear methods based on Euclidean geometry often fail to capture the true system behaviour, leading to trajectories that are dynamically inconsistent or infeasible.
In contrast, geometric methods provide an alternative formulation that more faithfully reflects the underlying nonlinear dynamics, for example by leveraging tools from differential geometry to analyze the mechanics of motor control~\cite{handzel1999geometric,flash2007affine,flash2018motor}.
Building on this perspective, Bullo and Lewis~\cite{bullo2019geometric} model the configuration space of a multi-linked system as a Riemannian manifold whose metric encodes the system's structural and dynamic properties.
Natural motions then follow geodesics that minimize intrinsic costs such as muscular effort~\cite{biess2007computational,sekimoto2009observation} or variations in joint torque~\cite{biess2011riemannian}.
Riemannian metrics have also been proposed to optimize additional criteria, including manipulability~\cite{jaquier2021geometry}, joint stiffness~\cite{saveriano2023learning}, and distance to kinematic singularities~\cite{maric2021riemannian}.
Motivated by these geometric insights, our approach frames the robot motion planning problem as a search for collision-free, minimum-length geodesics on Riemannian manifolds defined over high-dimensional configuration spaces.

To compute such minimum-length geodesics, one classical approach is to minimize an energy functional between two joint configurations, resulting in a boundary value problem defined by the Euler--Lagrange equations~\cite{peyre2010geodesic,shao2018riemannian}.
Although this yields accurate solutions, it is computationally expensive and often impractical in high-dimensional configuration spaces.
As a result, several approximate methods have been proposed to estimate geodesic distances on manifolds.
One representative approach solves the Eikonal equation using the Fast Marching Method, which estimates geodesic distances by propagating wavefronts over the manifold geometry~\cite{sethian1996fast,mirebeau2017anisotropic}.
Other techniques include the heat method~\cite{crane2013geodesics} and shortest-path computations on discretized meshes~\cite{surazhsky2005fast}.
Though less expensive than direct geodesic solvers, these methods rely on explicit manifold discretization or mesh-based representations, limiting their applicability to general high-dimensional robotic configuration spaces.

\looseness=-1
In contrast to techniques that rely on explicit discretization of the configuration space, sampling-based planners such as RRT~\cite{lavalle2001randomized} and RRT*~\cite{karaman2011sampling} scale well to high-dimensional spaces.
However, standard implementations typically rely on Euclidean metrics for distance computations, often ignoring the underlying geometry of the configuration space.
While constrained motion planning methods ensure feasibility by employing projection or retraction operators to maintain samples on implicitly defined manifolds~\cite{henderson2002multiple,csucan2012motion,jaillet2012path,kim2016tangent,kingston2018sampling,kingston2019exploring}, they generally still use the ambient metric from the configuration space (e.g., Euclidean distance) to measure the distances between samples.
Some methods approximate geodesic paths through sampling~\cite{salzman2013motion}, while others use Riemannian metrics to guide sampling~\cite{wilmarth1999motion}.
More recent approaches, such as RRT*-R, incorporate Riemannian metrics into planning on low-dimensional submanifolds~\cite{zhang2025rrt}.
Our method generalizes these approaches by integrating geometry-aware subroutines into sampling-based planners, ensuring that resulting paths are not only feasible with respect to manifold constraints but also asymptotically optimal with respect to the intrinsic Riemannian metric.

Several methods exploit geometric structure by designing or reshaping Riemannian metrics to encode constraints or task-relevant features.
Under this formulation, the cost of moving between configurations is expressed through a Riemannian metric on joint space, capturing the geometry of the configuration manifold~\cite{chen2016improved,laux2021robot,laux2023boundary}.
Metrics can also blend information from both joint and task spaces.
This is especially useful when dealing with task-space constraints such as obstacle avoidance.
Rather than treating these as cost terms or hard constraints, one can define metrics in task space that reflect obstacle geometry and pull them back to configuration space, yielding geodesics that naturally curve around obstacles~\cite{ratliff2015understanding,mainprice2016warping,mainprice2020interior,bao2025geodesic}.
This geometric reformulation has been shown to improve convergence in motion optimization methods such as CHOMP~\cite{ratliff2009chomp}.

\looseness=-1
Beyond reshaping metrics for single objectives, Riemannian geometry also provides a natural mechanism for blending several distinct motion objectives or constraints.
In \cite{ratliff2018riemannian,cheng2021rmpflow}, the authors propose combining multiple motion policies, each representing a different robot behaviour, into a single, geometrically consistent policy using Riemannian pullback operations.
Similar pullback operations have been used to reshape the configuration space metric based on barrier functions, allowing constraints like joint limits, self-collisions, and obstacles to be naturally encoded~\cite{beik2023reactive,klein2023design}.
In contrast to these approaches, we do not reshape the metric.
Instead, we work directly with the intrinsic Riemannian metric and minimize the curve length based on it, rather than minimizing the energy functional.
Constraints such as obstacle avoidance are handled implicitly through the sampling-based planner, without being explicitly encoded in the metric.
This approach allows us to leverage all the benefits of sampling-based planning while producing collision-free, minimum-length motions without requiring any handcrafting in the metric design.
\section{Preliminaries}
\label{sec:preliminaries}

\looseness=-1
This section presents a minimal set of tools from differential geometry required to follow the development in this work. For a more in-depth treatment of these concepts, we refer the reader to the following textbooks~\cite{lee2012smooth,lee2018introduction,boumal2023introduction}.

\subsection{Riemannian Metrics and Manifolds}
\label{subsec:riemannian_metrics_and_manifolds}

\looseness=-1
Let $\M$ be an $n$-dimensional manifold embedded in an ambient linear space $\mathcal{E}$ (e.g., $\M \subseteq \mathbb{R}^d$ with $d \ge n$).
By definition, $\M$ is a topological space that is locally Euclidean, meaning that each point $q \in \M$ has a neighborhood homeomorphic to an open subset of $\mathbb{R}^n$.
The tangent space $\TqM$ at a point $q \in \M$ consists of the velocity vectors of all smooth curves on $\M$ passing through $q$.
When $\M$ is embedded in $\mathcal{E}$, the tangent space $\TqM$ may be identified with a linear subspace of $\mathcal{E}$.
The collection of all tangent spaces forms the tangent bundle, defined as the disjoint union
\[
\TM = \bigl\{ (q, v) : q \in \M \; \text{and} \; v \in \TqM \bigr\}.
\]
Since each tangent space is a vector space, we can define an inner product on $\TqM$, given by a bilinear, symmetric and positive definite map ${\langle \cdot, \cdot \rangle}_q : \TqM \times \TqM \rightarrow \mathbb{R}$.
This inner product induces a norm $\smash{\Norm{v}_{q} = {\langle v, v \rangle}^{1/2}_{q}}\hspace{0.05em}$ on tangent vectors.
A Riemannian metric $G_{q}$ on $\M$ is a smoothly varying choice of such inner products for each $q \in \M$ acting on $\TM$.
In local coordinates, the Riemannian inner product between two tangent vectors $u, v \in \TqM$ can be written as
\[
{\langle u, v \rangle}_{q} = \Transpose{u} G_{q} v,
\]
where $G_q$ is a symmetric positive definite matrix representing the metric at $q$.
A manifold $\M$ equipped with a Riemannian metric is called a Riemannian manifold.

\subsection{Distances and Geodesics}
\label{subsec:distances_and_geodesics}

Given a Riemannian manifold $\M$, the metric induces a natural notion of length for smooth curves on $\M$, and consequently defines a distance function, making $\M$ into a metric space.
For a piecewise smooth curve $\pi : [0, 1] \rightarrow \M$, its length is defined as the integral of its speed under the Riemannian metric,
\begin{equation*}
\label{eqn:arc_length}
    L(\pi) = \int_0^1 \Norm{\dot{\pi}(t)}_{\pi(t)} \, dt
    = \int_0^1 \sqrt{ \Transpose{\dot{\pi}(t)} G_{\pi(t)} \, \dot{\pi}(t) } \, dt.
\end{equation*}
This induces the \emph{Riemannian distance} between two points on the manifold
\begin{equation}
\label{eqn:distance_function}
    d_{\M}(q_x, q_y) = \inf_{\pi} L(\pi),
\end{equation}
where the infimum is taken over all piecewise smooth curves $\pi$ such that $\pi(0) = q_x$ and $\pi(1) = q_y$.
Geodesics are curves that locally minimize this distance, generalizing the notion of straight lines in Euclidean space.
Equivalently, geodesics between two fixed endpoints minimize the curve energy
\begin{equation}
\label{eqn:energy}
E(\pi) = \frac{1}{2} \int_0^1 \Transpose{\dot{\pi}(t)} G_{\pi(t)} \, \dot{\pi}(t) \, dt.
\end{equation}
They satisfy the geodesic ordinary differential equation (ODE)
\begin{equation}
\label{eqn:geodesic_equation}
\ddot{q}^k(t) + \Gamma^{k}_{ij}(q(t)) \, \dot{q}^i(t) \, \dot{q}^j(t) = 0, \quad i,j,k \in \{1, \dots, n\},
\end{equation}
where $q(t) = ( q^{1}(t), \dots, q^{n}(t) )$ denotes the local coordinate representation of $\pi(t)$.\footnote{We use the Einstein summation convention here.}
Here, $\Gamma^{k}_{ij}$ are the Christoffel symbols of the second kind, computed directly from the Riemannian metric and given by
\begin{equation}
\label{eqn:christoffel}
\Gamma^{k}_{ij} = \frac{1}{2} \, G^{kl} \bigl( \frac{\partial G_{il}}{\partial q^{j}} + \frac{\partial G_{jl}}{\partial q^{i}} - \frac{\partial G_{ij}}{\partial q^{l}} \bigr),
\end{equation}
which describe how the coordinate basis varies across the manifold.
Computing geodesics from (\ref{eqn:geodesic_equation}) requires solving a boundary value problem, which becomes computationally challenging in high-dimensional spaces.
For this reason, rather than solving the geodesic ODEs directly, we consider minimizing the curve length in (\ref{eqn:distance_function}) under the intrinsic Riemannian metric (see Section~\ref{subsec:motion_planning_problem_formulation}).

\subsection{Motion Planning Problem Formulation}
\label{subsec:motion_planning_problem_formulation}

Let $\mathcal{Q} \subseteq \M$ denote the configuration manifold (for example, a constrained lower-dimensional submanifold), with each $q \in \mathcal{Q}$ representing a configuration of the system.
In this work, we restrict attention to the unconstrained case $\mathcal{Q} = \M$, while retaining the notation $\mathcal{Q}$ for generality.
Let $\mathcal{Q}_{\text{obs}} \subsetneq \mathcal{Q}$ denote the set of configurations in collision with obstacles.
The obstacle-free configuration space is defined as $\mathcal{Q}_{\text{free}} = \mathcal{Q} \setminus \mathcal{Q}_{\text{obs}}$.
Given an initial configuration $q_{\text{start}}$ and a goal configuration $q_{\text{goal}}$, the objective is to find a feasible path $\pi^*$ in $\mathcal{Q}_{\text{free}}$ that minimizes the Riemannian length objective given in  (\ref{eqn:distance_function}):
\begin{equation}
\label{eqn:optimal_planning}
\pi^* = \ArgMin{\pi \in \Sigma} 
\Big\{ L(\pi) \;\Big|\; 
\pi(0)=q_{\text{start}},\ 
\pi(1)=q_{\text{goal}},\ 
\pi(t) \in \mathcal{Q}_{\text{free}},\ \forall t \in [0,1] 
\Big\},
\end{equation}
where $\Sigma$ denotes the set of all piecewise smooth feasible paths.

\looseness=-1
In this work, we address the planning problem in (\ref{eqn:optimal_planning}) by adopting sampling-based motion planning methods.
Specifically, we search for solutions by incrementally constructing a rapidly-exploring random tree through random sampling, and refining it via incremental rewiring to asymptotically approach shortest paths in $\mathcal{Q}_{\text{free}}$ \cite{lavalle2001randomized,karaman2011sampling}.
This strategy enables scalability to high-dimensional configuration spaces while respecting the intrinsic geometry of the manifold.
Because the Riemannian metric generally distorts the space, shortest paths are no longer straight lines in the ambient space $\mathcal{E}$ (e.g., $\mathbb{R}^d$), and it is therefore critical that distance computations and interpolations within the planner remain consistent with the underlying manifold structure.
We address geometric consistency in Section~\ref{sec:method}.
\section{Geodesic Motion Planning on Riemannian Manifolds}
\label{sec:method}

This section develops the geometry-aware subroutines that enable sampling-based motion planning algorithms to find geodesic paths on Riemannian manifolds.
We first introduce a computationally efficient approximation of geodesic distance between configurations (Section~\ref{sec:distance-between-configurations}).
We then describe a gradient-based interpolation procedure that uses this distance to extend the search tree while respecting the underlying geometry of the manifold (Section~\ref{sec:vertex-expansion}).

\subsection{Distance Between Configurations}
\label{sec:distance-between-configurations}

The distance between two configurations on a Riemannian manifold is defined as the length of the shortest connecting curve under the intrinsic metric~\eqref{eqn:distance_function}.
Exact evaluation of this distance requires solving a geodesic boundary value problem, which is impractical for online use due to the repeated calls required by nearest-neighbour queries and tree rewiring in sampling-based planners.
Accordingly, we approximate the geodesic distance by evaluating the Riemannian metric at the midpoint of the two configurations and computing the length of a piecewise path in the midpoint tangent space.
This approximation is computationally `cheap' since it requires only a single metric evaluation per subroutine call.
We later show that this midpoint-based approximation converges to the true geodesic distance as the configurations become arbitrarily close to each other (Theorem~\ref{thm:accuracy}).

\begin{figure}[!tb]
\centering
\vspace{-0.4\baselineskip}
\includegraphics[width=0.95\textwidth]{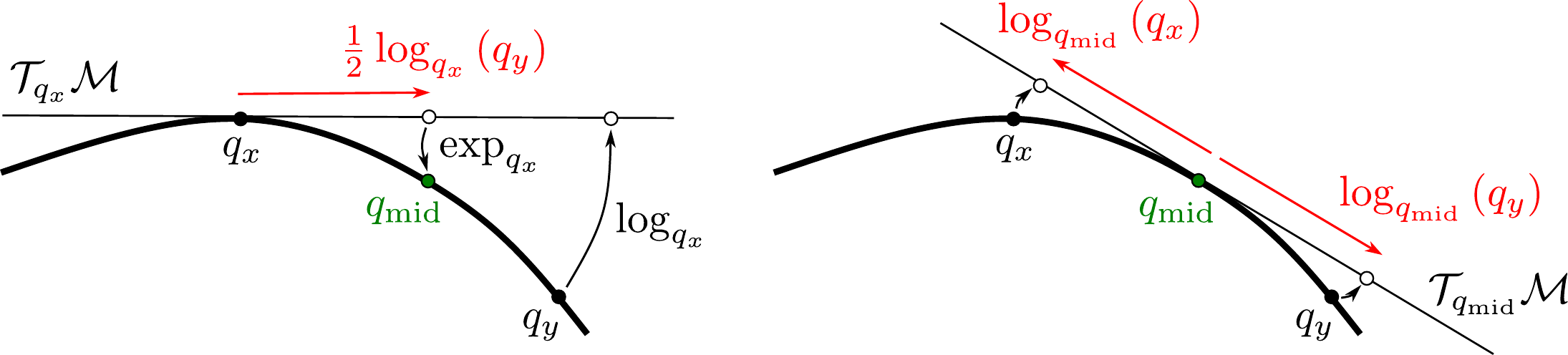}
\vspace{-0.5\baselineskip}
\caption{
\looseness=-1
Midpoint-based geodesic distance between configurations $q_x$ and $q_y$ on the manifold $\M$.
The geodesic midpoint $q_{\textrm{mid}}$ is constructed by interpolating halfway along the tangent vector connecting $q_x$ and $q_y$ in either of their tangent spaces and mapping the result back to the manifold (left).
The distance is then computed in the tangent space at $q_{\textrm{mid}}$ as the Riemannian norm of the difference between corresponding tangent vectors (right).
}
\vspace{-1.25\baselineskip}
\label{fig:midpoint-distance}
\end{figure}

We begin by analyzing the geometry of the geodesic midpoint and establish that this formulation yields an exact distance identity.
\begin{definition}[Geodesic Midpoint]
\label{def:geodesic-midpoint}
\looseness=-1
Let $ \mathcal{M}$ be a Riemannian manifold and let $q_x, q_y \in \mathcal{M}$ denote two configurations (points) on the manifold.
Assuming $q_x$ and $q_y$ lie within a geodesically convex neighbourhood, the geodesic midpoint $q_{\mathrm{mid}}$ is constructed by mapping half the geodesic distance from $q_x$ towards $q_y$
\begin{equation*}
    q_{\mathrm{mid}} = \exp_{q_x} \left(\frac{1}{2} \log_{q_x} \left(q_y \right) \right).
\end{equation*}
By \cite[Theorem 6.17]{lee2018introduction}, $q_{\mathrm{mid}}$ lies on the unique minimizing geodesic $\pi:[0,1]\to\mathcal{M}$ from $q_x$ to $q_y$ such that $\pi(0)=q_x$, $\pi(1)=q_y$, and $\pi(1/2)=q_{\mathrm{mid}}$.
\end{definition}
This construction is particularly useful since it provides a symmetric reference frame (Figure~\ref{fig:midpoint-distance}).
As the following lemma shows, projecting both configurations into the tangent space of the midpoint yields a symmetric distance relation.
\begin{lemma}
The Riemannian distance between $q_x$ and $q_y$ satisfies the identity
\begin{equation}
\label{eqn:midpoint-distance}
d_{\mathcal{M}}(q_x, q_y) = \left\| \log_{q_{\mathrm{mid}}}(q_y) - \log_{q_{\mathrm{mid}}}(q_x) \right\|_{q_{\mathrm{mid}}},
\end{equation}
where $\log_{q_{\mathrm{mid}}}(\cdot)$ denotes the Riemannian logarithmic map on the tangent space $\mathcal{T}_{q_{\textrm{mid}}}\M$ and $\Norm{\cdot}_{q_{\mathrm{mid}}}$ denotes the norm induced by the metric at the midpoint.
\end{lemma}
\begin{proof}
By Definition~\ref{def:geodesic-midpoint}, $\pi$ is the unique minimizing geodesic connecting $q_x$ to $q_y$ within a geodesically convex neighbourhood on $\M$.
Since $\pi$ is a geodesic, the curve segment connecting $q_{\mathrm{mid}}$ to $q_y$ is also a geodesic, given by the reparameterization $\pi_y(t) = \pi(\tfrac{1}{2} + \tfrac{t}{2})$ for all $t\in[0,1]$.
The value of the Riemannian logarithmic map at $q_{\mathrm{mid}}$ is precisely the initial velocity of this segment
\begin{equation*}
\log_{q_{\mathrm{mid}}}(q_y) = \dot\pi_y(0) = \frac{1}{2}\,\dot\pi(1/2).
\end{equation*}
Similarly, the segment connecting $q_{\mathrm{mid}}$ to $q_x$ is given by $\pi_x(t) = \pi(\tfrac{1}{2} - \tfrac{t}{2})$, yielding
\begin{equation*}
\log_{q_{\mathrm{mid}}}(q_x) = \dot\pi_x(0) = -\frac{1}{2}\,\dot\pi(1/2).
\end{equation*}
Subtracting the two tangent vectors yields
\begin{equation*}
\log_{q_{\mathrm{mid}}}(q_y) - \log_{q_{\mathrm{mid}}}(q_x) =
\frac{1}{2}\,\dot\pi(1/2) - \left(-\frac{1}{2}\,\dot\pi(1/2)\right) =
\dot\pi(1/2).
\end{equation*}
Taking the norm with respect to the Riemannian metric at $q_{\mathrm{mid}}$, we obtain
\begin{equation*}
\Norm{\log_{q_{\mathrm{mid}}}(q_y) - \log_{q_{\mathrm{mid}}}(q_x)}_{q_{\mathrm{mid}}} =
\Norm{\dot\pi(1/2)}_{q_{\mathrm{mid}}}.
\end{equation*}
Since $\pi$ is a minimizing geodesic, it has constant speed, implying that $\Norm{\dot\pi(t)}_{\pi(t)}$ is constant for all $t$.
Consequently,
\begin{equation*}
\Norm{\dot\pi(1/2)}_{q_{\mathrm{mid}}}=
\int_0^1 \Norm{\dot{\pi}(t)}_{\pi(t)}\,dt=
d_{\M}(q_x, q_y)
\tag*{\qed}
\end{equation*}
\end{proof}

In practice, computing exponential and logarithmic maps requires solving a geodesic boundary value problem, which is typically too computationally expensive for use in sampling-based motion planners.
In this work, we propose replacing these operations with \textit{retractions}, which are first-order approximations of the exponential map, and their local inverses, which serve as computationally less expensive approximations of the logarithmic map.

\begin{definition}[Retraction]
\label{def:retractions}
A retraction on a Riemannian manifold $\M$ is a smooth map
$\R : \TM \rightarrow \M; \, (q, v) \mapsto \R_{q}(v)$
from the tangent bundle $\TM$ onto $\M$ such that its restriction $\R_{q} : \TqM \rightarrow \M$ to the tangent space at $q$ satisfies $\R_{q}(0) = q$ and the identity map $\D \R_{q}(0)[v] = v$.
\end{definition}
Since the differential $\D \R_{q}(0)$ is nonsingular, the inverse function theorem guarantees that $\R_{q}$ is a local diffeomorphism around the origin.
Consequently, it admits a local inverse $\mathcal{R}_q^{-1}$ defined in a neighbourhood of $q$.
Leveraging these operators, we formulate an approximation of the midpoint distance in \eqref{eqn:midpoint-distance} by substituting the $\exp$ and $\log$ maps with their retraction counterparts,
\begin{equation}
\label{eqn:approx-midpoint-distance}
\hat{d}_{\mathcal{M}}(q_x, q_y) = \left\| \R^{-1}_{\hat{q}_{\mathrm{mid}}}(q_y) - \R^{-1}_{\hat{q}_{\mathrm{mid}}}(q_x) \right\|_{\hat{q}_{\text{mid}}},
\end{equation}
where the retraction midpoint $\hat{q}_{\text{mid}}$ can be constructed as
\begin{equation}
\label{eqn:retraction-midpoint}
\hat{q}_{\text{mid}} = \R_{q_x} \left(\frac{1}{2} \R^{-1}_{q_x} \left(q_y \right) \right).
\end{equation}

The primary advantage of the construction in~\eqref{eqn:approx-midpoint-distance} is its higher-order local accuracy relative to endpoint-based retraction distances, analogous to the improved accuracy of central finite differences over forward or backward differences in numerical analysis.
Although the retraction midpoint only approximates the true geodesic midpoint, it is accurate enough that the first- and second-order distortion terms introduced by first-order retractions cancel in the final difference, effectively yielding an extra order of accuracy for free.
Theorem~\ref{thm:accuracy} formalizes this argument, establishing that the leading term in the Taylor expansion of the approximation error is third order in the separation between configurations.

\begin{theorem}
\label{thm:accuracy}
Let $\mathcal{M}$ be a smooth Riemannian manifold and $U \subseteq \mathcal{M}$ be a compact subset.
For any configurations $q_x, q_y \in U$, the midpoint retraction distance approximates the true Riemannian distance with cubic accuracy, that is, the approximation error satisfies
\begin{equation*}
\left| \hat{d}_{\mathcal{M}}(q_x, q_y) - d_{\mathcal{M}}(q_x, q_y) \right| = \mathcal{O}\left(d_{\mathcal{M}}(q_x, q_y)^3\right).
\end{equation*}
\end{theorem}
\begin{proof}
The result follows from the Taylor expansion of the metric tensor in Riemann normal coordinates centred at the midpoint; see Appendix~\ref{sec:appendix:proofs}. \qed
\end{proof}

\subsection{Vertex Expansion}
\label{sec:vertex-expansion}

Sampling-based motion planning methods typically require a fast and efficient
procedure for connecting two configurations, a process often referred to as
local planning.
In the case of RRTs, for example, this operation attempts to connect the nearest
neighbour to a randomly sampled configuration using an edge of bounded length,
thereby generating a new candidate configuration.
In Euclidean spaces, such connections are typically implemented using straight-line interpolation.
On curved configuration spaces, however, this task requires finding geodesics
that locally minimize distance while respecting the intrinsic geometry of the
manifold.

\looseness=-1
At a high level, our method performs this interpolation by constructing a discrete approximation to a geodesic that follows the natural gradient of the squared Riemannian distance potential.
Rather than solving for a continuous curve in closed form, we generate a sequence of configurations by iteratively descending along this gradient.
Such descent requires moving around the manifold along a specified direction, for which we make use of retractions (Definition~\ref{def:retractions}).
A retraction maps a configuration $q \in \M$ and a tangent vector $v \in \TqM$ to a new point on the manifold, thereby allowing local movement while remaining on $\M$.\footnote{For example, on a linear manifold, the mapping $\R_q(v) = q + v$ is a valid retraction.}
We next formalize the notion of gradients on a Riemannian manifold (Definition~\ref{def:riemannian-gradient}).

\begin{definition}[Riemannian gradient]
\label{def:riemannian-gradient}
Let $\phi : \M \rightarrow \mathbb{R}$ be a smooth function on a Riemannian manifold $\M$. The Riemannian gradient of $\phi$ is the vector field $\grad \phi$ on $\M$ uniquely defined by
\begin{equation*}
    \D \phi(q)[v] = \left\langle v, \grad \phi(q) \right\rangle_{q} \quad \forall v \in \TqM,
\end{equation*}where $\left\langle \cdot, \cdot \right\rangle_{q}$ is the inner product induced by the Riemannian metric $G$ on $\TqM$ and $\D \phi(q)[v]$ is the differential of $\phi$ at $q$ along $v$.
\end{definition}
In practice, the Riemannian gradient of $\phi$ can be computed by evaluating a classical gradient at the origin of the linear tangent model provided by the retraction.
Specifically, for any retraction $\R$ on $\M$,
\begin{equation*}
\grad \phi(q) = \grad(\phi \circ \R_{q})(0),
\end{equation*}where $\phi \circ \R_{q} : \TqM \rightarrow \mathbb{R}$ is a smooth function defined on the linear space $\TqM$ endowed with the inner product $\langle \cdot, \cdot \rangle_{q}$ \cite{boumal2023introduction}.
\looseness=-1
Equivalently, in local coordinates,
\begin{equation}\label{eqn:grad-local}
\grad \phi(q) = \Inv{G(q)} \nabla_{u} (\phi \circ \R_{q})(0)
\end{equation}
where $u\in\mathbb{R}^n$ are coordinates on $\TqM$ and $\nabla_{u}$ denotes the usual Euclidean gradient on $\TqM$ expressed in any basis.
Using \eqref{eqn:grad-local}, the \textit{Riemannian natural gradient} direction at iteration $k$ is $v_k = \grad \phi(q_k)$.
An update rule for taking a single discrete step along the direction of steepest descent on $\M$ can then be written as
\begin{equation*}
    q_{k+1} = \R_{q_k} \left( - s_k \, \hat{v}_k \right), \quad \hat{v}_k = \frac{v_k}{\|v_k\|_{q_k}}
\end{equation*}
with a step length $s_k > 0$, which may be chosen by any backtracking test.
Here, we normalize $v_k$ to ensure the step size $s_k$ approximates the arc length along the retraction curve, and use the distance function in (\ref{eqn:approx-midpoint-distance}) as our squared cost potential.
Specifically, for any fixed configuration, say $q^\dagger$, we define the function
\begin{equation}\label{eqn:potential}
\phi(q) = \frac{1}{2}\, \hat{d}_{\M}(q, q^\dagger)^2
\end{equation} and compute the gradient of (\ref{eqn:potential}) in local coordinates using (\ref{eqn:grad-local}).

\begin{algorithm}[!tb]
\caption{Expansion $( \qnear, \qrand \in \mathcal{M} )$}\label{algo:expansion}
$q \gets q_{\textrm{near}}$, $d \gets 0$\;

\Repeat{$ \hat{d}_{\M}\,(q, \qrand) \leq s $}{

$u \gets \Inv{\R}_{q}(\qrand)$\label{algo1:line3}\;
$v \gets {\Inv{G(q)}} \nabla_u \, \left(\phi \circ \R_{q}\right) \left( 0 \right) $\label{algo1:line4}\;

$q_{\textrm{next}} \gets \R_{q}\left(- s \, \hat{v}\right)$\label{algo1:line6}\;

\If{$ \hat{d}_{\M}\,(q, q_{\textrm{next}}) > \lambda \, s $\label{algo1:line7}}{
    $s \gets \frac{1}{2}\, s$\;
    \If{$s < s_{\textrm{min}}$}{\Break}
    \Continue\label{algo1:line11}
}

$d \gets d + \hat{d}_{\M}\,(q, q_{\textrm{next}})$\label{algo1:line12}\;
\If{$d > d_{\textrm{max}}$}{\Break\label{algo1:line14}}

$q \gets q_{\textrm{next}}$\label{algo1:line15}\;
}

\Return{q}\;
\end{algorithm}

\looseness=-1
The complete vertex expansion procedure is presented in Algorithm~\ref{algo:expansion}.
The algorithm performs the expansion process by repeatedly taking small retraction steps on the manifold starting from the nearest configuration $\qnear$ until the extended branch is close to the random configuration $\qrand$.
At each iteration, it maps $\qrand$ to the current tangent space at $q$ via the inverse retraction, yielding a tangent vector $u \in \TqM$ expressed in local coordinates (Algorithm~\ref{algo:expansion}, Line~\ref{algo1:line3}).
It then evaluates the gradient of~(\ref{eqn:potential}) in the current tangent space and computes the negative natural-gradient direction $v$ with respect to the Riemannian metric (Algorithm~\ref{algo:expansion}, Line~\ref{algo1:line4}).
Taking a step size $s$ along the normalized direction $\hat{v}$ using retraction yields a new configuration $q_{\text{next}}$ (Algorithm~\ref{algo:expansion}, Line~\ref{algo1:line6}).
If the resulting displacement exceeds a threshold proportional to the step size, specifically a factor $\lambda s$, the step length is halved (Algorithm~\ref{algo:expansion}, Lines~\ref{algo1:line7}--\ref{algo1:line11}).
We also track the cumulative distance traveled, denoted $d$, and enforce an upper bound $d_{\text{max}}$ to avoid unnecessarily long expansions, which may otherwise arise due to manifold curvature or variations in the Riemannian metric (Algorithm~\ref{algo:expansion}, Lines~\ref{algo1:line12}--\ref{algo1:line14}).
If none of the stopping conditions is triggered, the step is accepted and the current configuration is updated accordingly (Algorithm~\ref{algo:expansion}, Line~\ref{algo1:line15}).
\section{Experiments}
\label{sec:results}

We validate our proposed approach through two use cases to demonstrate the versatility of geometry-aware sampling-based planners for robot motion planning problems.
First, we evaluate the planner's ability to find minimum-energy motions for serial-link manipulators (Section~\ref{sec:exp-manipulators}).
Second, we apply the method to rigid-body planning on $\LieGroupSE{2}$ under nonholonomic constraints (Section~\ref{sec:exp-se2-planning}).
We compare against standard numerical optimization baselines, including a boundary value problem (BVP) solver and a variational energy minimization method, as well as a sampling-based planner using a Euclidean metric.
To enforce obstacle avoidance in the optimization-based methods, we reshape the metric using standard exponential barrier functions and tune the corresponding parameters to achieve the best performance possible.
To account for the sensitivity of variational solvers to initial conditions and the stochasticity of sampling-based methods, we conduct all experiments over $50$ 
trials with start and goal configurations perturbed by Gaussian noise, running each sampling-based planner $10$ times per trial.
In all experiments, we use the Open Motion Planning Library (OMPL)~\cite{sucan2012open} implementation of RRT* as the underlying sampling-based planner.\footnote{Code is publicly available at \url{https://github.com/utiasSTARS/geodex}.}
For geodesic computation with the variational solvers, we use the StochMan Python library~\cite{software:stochman} to represent geodesics as cubic splines with a fixed number of control points and optimize them to minimize the Riemannian energy functional, following~\cite{beik2023reactive, klein2023design}.

Throughout our analysis, we evaluate path quality using \textit{geodesic length} and \textit{energy}, both measured under the Riemannian metric specific to each experiment.
To ensure consistent reporting, we reparameterize all solution paths to unit-speed curves prior to evaluation and then report the geometric length (which is invariant to parameterization) and the Dirichlet energy functional defined in \eqref{eqn:energy}, which quantifies the smoothness of the geodesic.

\subsection{Serial Link Manipulators}
\label{sec:exp-manipulators}

\looseness=-1
This section presents experimental results assessing the effectiveness of our approach for motion planning with configuration-dependent Riemannian metrics for serial manipulators.
In all experiments, we define the Riemannian metric using the manipulator's mass-inertia matrix~\cite{jaquier2022riemannian}, computed via the composite rigid body algorithm implemented in the Pinocchio library~\cite{carpentier2019pinocchio}.
Under this metric, geodesics correspond to motions that minimize kinetic energy for a given traversal time.

\begin{figure}[!t]
\centering
\vspace{-0.5\baselineskip}
\includegraphics[width=0.9\textwidth]{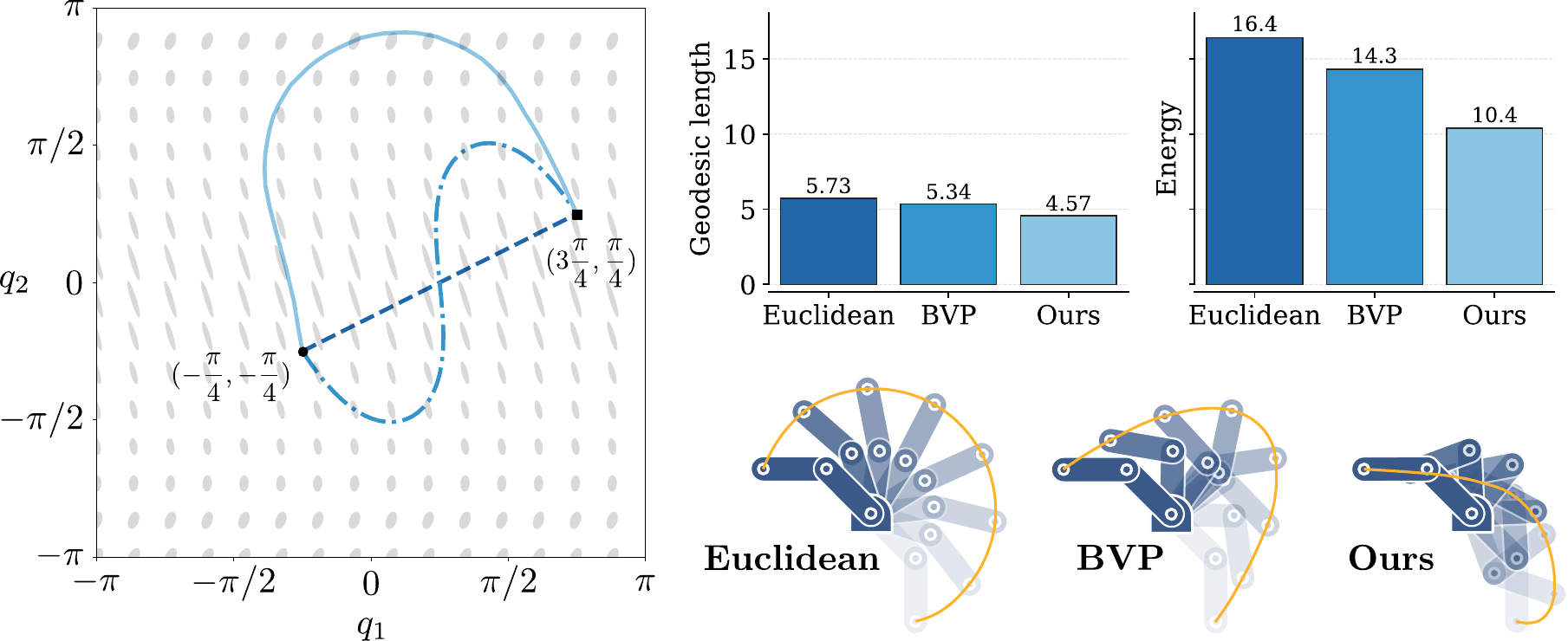}
\vspace{-0.5\baselineskip}
\caption{
\looseness=-1
Geodesics found by various motion planning methods for the $2$-DoF planar manipulator experiment in Section~\ref{sec:exp-manipulators}.
Configuration space paths from start ($\bullet$) and goal (\scalebox{0.6}{$\blacksquare$}) are shown (left), with shaded regions indicating kinetic-energy ellipsoids at different configurations.
Corresponding task space motions with end-effector trajectories are shown in yellow (right).
The Euclidean approach minimizes joint space shortest distance, yielding straight-line geodesics (\inlinedashedblue).
Numerical methods optimize the energy functional, producing lower-cost paths (\inlinedasheddottedblue) but often converge to local minima.
Our geometry-aware sampling-based approach recovers globally optimal geodesics under the intrinsic kinetic-energy Riemannian metric (\inlinesolidblue), achieving lower geodesic length and energy.
}
\vspace{-1.35\baselineskip}
\label{fig:two-link-planar-arm}
\end{figure}

\subsubsection{Two-Link Planar Arm}
We first consider a simple 2-DoF planar manipulator in an obstacle-free environment.
The arm consists of two identical links, each with a length of $1.0$ m and a mass of $1.0$ kg\footnote{Links are modelled as uniform slender rods with centres of mass at the midpoints.}.
The task is to plan a motion from an initial configuration of $\smash{q_{\textrm{start}}=\Transpose{[-\pi/4, -\pi/4]}}$ to a goal configuration of $\smash{q_{\textrm{goal}}=\Transpose{[3\pi/4, 3\pi/4]}}$.
Unlike the Euclidean baseline, which yields a straight-line path in configuration space, our approach correctly recovers the curved Riemannian geodesic (Figure~\ref{fig:two-link-planar-arm}).
Even with identical link masses, the effective inertia at the base joint is higher than at the elbow joint; our planner naturally captures this property, generating trajectories that minimize the mechanical work.
We also observe that classical numerical baselines (both BVP and variational solvers) are highly sensitive to the initial guess; when initialized with a standard straight-line path, they frequently converge to suboptimal solutions.
In contrast, our geometry-aware sampling-based planner consistently recovers the globally optimal geodesic without requiring any prior knowledge of the solution geometry, successfully navigating the nonlinear energy landscape where optimization methods struggle.
This behavior highlights the strong dependence of classical numerical methods on initialization in highly nonconvex energy landscapes.
By contrast, the sampling-based formulation instead leverages geometric structure to explore multiple homotopy classes, enabling reliable convergence to the global geodesic without requiring problem-specific initialization.

\subsubsection{7-DoF Franka}
To evaluate scalability in high-dimensional spaces, we apply our planner to a 7-DoF Franka Emika robot operating in a cluttered environment.
We utilize the \textit{table pick} environment from the MotionBenchMaker dataset~\cite{chamzas2021motionbenchmaker}, where the robot must navigate from a start configuration to a grasp pose while avoiding collisions with the table and surrounding obstacles (Figure~\ref{fig:franka}).
Because of the dimensionality and complexity of this problem, we exclude the BVP solver from the set of baselines due to its poor scaling behavior.
Table~\ref{tab:results} summarizes the results over 50 trials.
From the data, we observe that all methods achieved high success rates in this setting.
Since we employ a single-tree RRT* as the underlying sampling-based planner, some trials fail to reach the goal due to sampling stochasticity; this limitation could be mitigated by adopting a bidirectional search strategy~\cite{kuffner2000rrt}, for example.
While the Euclidean approach reliably finds feasible paths, it does not account for the robot's configuration-dependent inertia, often producing high-energy motions that unnecessarily excite the heavy base joints.
Conversely, although the variational solver explicitly minimizes energy, the reshaping of the metric induces a complex, non-convex landscape, making the method sensitive to local minima and requiring careful tuning of barrier parameters.
In contrast, our geometry-aware planner consistently produces lower kinetic-energy paths than both baselines, effectively handling obstacle avoidance constraints implicitly without the need for explicit metric design.

\begin{table}[!t]
\centering
\vspace{-0.9\baselineskip}
\caption{
\looseness=-1
Performance statistics over 50 trials for the 7-DoF Franka arm and anisotropic $\LieGroupSE{2}$ planning tasks described in Sections~\ref{sec:exp-manipulators} and \ref{sec:exp-se2-planning}.
We report median values for geodesic length and energy measured under the corresponding Riemannian metrics; success denotes the fraction of trials that produced a collision-free path.
While timing information is not shown here, achieving the indicated performance requires at least 2 to 3 minutes for variational methods, whereas our approach is limited to just one minute.
}
\medskip
\label{tab:results}
\renewcommand{\arraystretch}{1.3}
\setlength{\tabcolsep}{6pt}
\resizebox{0.85\columnwidth}{!}{%
\renewrobustcmd{\bfseries}{\fontseries{b}\selectfont}
\renewrobustcmd{\boldmath}{}
\begin{tabular}{
    lr | 
    S[table-format=2.1(2),detect-weight,detect-display-math,mode=text]
    S[table-format=4.1(4),detect-weight,mode=text]
    S[table-format=2.0,detect-weight,mode=text]<{\%}
    }
    \toprule
    & & {Length ($\downarrow$)} & {Energy ($\downarrow$)} & \multicolumn{1}{c}{Success ($\uparrow$)} \\
    \midrule
    
    % --- Franka ---
    & Variational & 2.5(0.6) & 3.1(1.5) & \bfseries 96 \\
    & Sampling (Euclidean) & 2.6(0.5) & 3.5(1.5) & 85 \\
    \multirow{-3}{*}{\textbf{Franka}} 
    & Sampling (Ours) & \bfseries 2.1(0.2) & \bfseries 2.3(0.4) & 90 \\
    
    \midrule
    \midrule
    
    % --- SE(2) Doorway ---
    & Variational & 24.9(0.6) & 310.4(14.0) & 86 \\
    & Sampling (Euclidean)   & 43.7(3.7) & 954.4(174.7) & \bfseries 100 \\
    \multirow{-3}{*}{$\boldsymbol{\LieGroupSE{2}}$ \textbf{Doorway}} 
    & Sampling (Ours) & \bfseries 23.2(0.5) & \bfseries 269.1(11.7) & \bfseries 100 \\
    
    \midrule
    
    % --- SE(2) Corridor ---
    & Variational & {$\infty$} & {$\infty$} & 8 \\
    & Sampling (Euclidean)   & 95.6(9.2) & 4571.2(873.0) & \bfseries 100 \\
    \multirow{-3}{*}{$\boldsymbol{\LieGroupSE{2}}$ \textbf{Corridor}}
    & Sampling (Ours) & \bfseries 43.0(0.5) & \bfseries 925.9(22.5) & \bfseries 100 \\
    
    \bottomrule
\end{tabular}
}
\vspace{-1\baselineskip}
\end{table}

\subsection{Planning on $\LieGroupSE{2}$}
\label{sec:exp-se2-planning}

\looseness=-1
We next apply our framework to rigid-body planning on the special Euclidean group $\LieGroupSE{2}$ to demonstrate how metric design can enforce specific kinematic behaviors without explicitly encoding them as nonholonomic constraints, as is common in kinodynamic motion planners~\cite{donald1993kinodynamic,laumond2002motion}.
We evaluate our approach on a large-scale navigation environment from the literature, the Willow Garage map~\cite{quigley2009ros}, considering two distinct scenarios, \textit{Doorway} and \textit{Corridor} (Figure~\ref{fig:se2-path}).
To examine the effect of metric shaping, we employ left-invariant Riemannian metrics defined by a diagonal weight matrix in the body frame,
$G=\text{diag}(w_x, w_y, w_{\theta})$, which assigns independent costs to longitudinal translation, lateral translation, and rotation~\cite{belta2002euclidean}.
In this experiment, we penalize lateral translation by making the metric highly anisotropic (i.e., $w_y \gg w_x$), effectively creating a `soft' nonholonomic constraint.

As reported in Table~\ref{tab:results}, both the Euclidean baseline and our proposed method achieve a 100$\%$ success rate in finding collision-free paths within the available planning time.
The variational method, however, often struggles with the complex geometry of the environment, particularly in the \textit{Corridor} scenario, where it attains only an $8\%$ success rate.
This behavior reflects the sensitivity of variational solvers to initialization and the difficulty of tuning barrier-function weights to navigate narrow passages without becoming trapped in local minima or violating collision constraints.
Although the Euclidean planner reliably finds feasible paths, the quality of its solutions is poor, incurring substantially higher geodesic length and energy costs compared to our method.
As shown in Figure~\ref{fig:se2-path}, the geodesics found by the Euclidean planner ignore the anisotropic nature of the underlying metric, resulting in unnatural skidding or screw motions where the rigid body translates laterally.
Conversely, our geometry-aware planner naturally recovers the soft nonholonomic behavior implied by the intrinsic Riemannian metric, aligning the orientation with the direction of travel to minimize energy and producing significantly shorter geodesics than the baselines.

\begin{figure}[!tb]
\centering
\vspace{-0.5\baselineskip}
\includegraphics[width=0.78\textwidth]{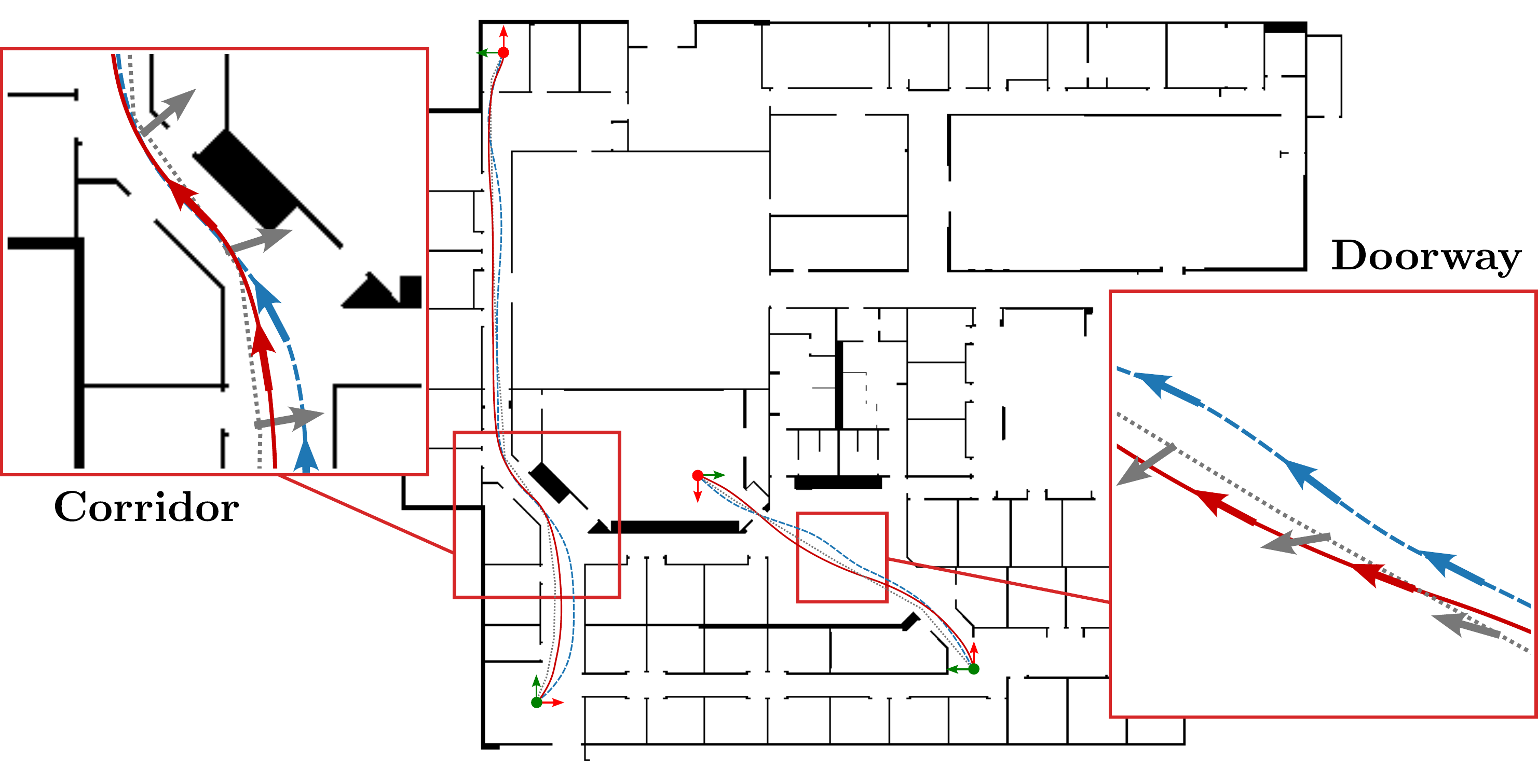}
\vspace{-0.8\baselineskip}
\caption{
\looseness=-1
Comparison of collision-free geodesics produced by the benchmarked methods in the Willow Garage environment for the $\LieGroupSE{2}$ planning experiment described in Section~\ref{sec:exp-se2-planning}.
The proposed method (\inlinesolidred) is compared against the variational method (\inlinevariationalblue) and a sampling-based planner using a Euclidean metric (\inlinedottedgray).
Arrows along each path represent the $\LieGroupSE{2}$ pose at discrete intervals; for visual clarity, only the body-frame $x$-axis is shown to illustrate orientation along the trajectory.
Start and goal configurations are indicated by green and red dots, respectively, with their full orientation frames.
}
\vspace{-1.35\baselineskip}
\label{fig:se2-path}
\end{figure}
\section{Conclusion}
\label{sec:conclusion}

This work presents a geometry-aware sampling-based planning framework for robot configuration spaces equipped with a Riemannian metric.
Our main contribution is a midpoint-based distance approximation on Riemannian manifolds that can be evaluated using only retractions and local metric information, yet matches the true Riemannian distance with third-order accuracy.
This approximation makes the distance-based subroutines of sampling-based planners more consistent with the intrinsic motion costs than those of Euclidean baselines, while avoiding the expense of solving geodesic boundary-value problems.
We further show how the same ingredients enable a geometry-aware local interpolation approach based on discrete retraction steps and Riemannian natural gradients.
Across manipulation and $\LieGroupSE{2}$ problems under anisotropic metrics, our method consistently produces higher-quality solutions under the target metric, especially in settings where Euclidean or isotropic assumptions are misleading.
Future work will explore the full implications of our approach.
In particular, we plan to investigate the design of heuristic functions in curved spaces to focus search on promising regions of the configuration space.
We are also interested in studying tighter theoretical guarantees when our proposed geometry-aware subroutines are used within asymptotically optimal planners.

\appendix
\section{Proofs}
\label{sec:appendix:proofs}

\looseness=-1
We first define the notation and assumptions used throughout.
Let $q_x, q_y \in \M$ be two configurations within a geodesically convex neighbourhood.
Then, there exists a unique minimizing geodesic $\pi : [0,1] \to \M$ with $q_x = \pi(0)$ and $q_y = \pi(1)$.
We denote the geodesic midpoint by $q_{\mathrm{mid}} = \pi(1/2)$ and the distance by $h = d_{\M}(q_x, q_y)$.
We define $u = \log_{q_{\mathrm{mid}}}(q_y) \in \mathcal{T}_{q_{\mathrm{mid}}}\M$ as the tangent vector at $q_{\mathrm{mid}}$ pointing toward $q_y$.
Consequently, $\Norm{u}_{q_{\mathrm{mid}}} = d_{\M}(q_{\mathrm{mid}}, q_y) = h/2$.
By the symmetry of the minimizing geodesic, we have $\log_{q_{\mathrm{mid}}}(q_x) = -u$.
For convenience, we utilize Riemann normal coordinates centred at $q_{\mathrm{mid}}$.
In these coordinates, the metric tensor at the origin is the identity matrix $I$, and its first-order partial derivatives vanish.
Accordingly, the coordinates of $q_x$ and $q_y$ are given by $-u$ and $u$, respectively.

\begin{lemma}
\label{lemma:retraction-midpoint}
Let $\delta \in \mathbb{R}^n$ be the coordinate representation of the retraction midpoint $\hat{q}_{\mathrm{mid}}$.
Then $\delta$ approximates the origin with second-order accuracy.
\end{lemma}

\begin{proof}
\looseness=-1
Since $\R_q(0)=q$ and $\D\R_q(0)=I$, the coordinate expansion is
\begin{equation}
\label{eqn:lemma-retraction}
\R_{q}(v) = q + v + \mathcal{O}(\Norm{v}^2).
\end{equation}
Consequently, the inverse retraction satisfies
\begin{equation}
\label{eqn:lemma-inverse-retraction}
\Inv{\R}_q(p) = (p-q) + \mathcal{O}(\Norm{p-q}^2),
\end{equation}
for any $p \in \M$ sufficiently close to $q$.
Let $v = \Inv{\R}_{q_x}(q_y) \in \mathcal{T}_{q_x}\M$.
Substituting the coordinate representations $q_x = -u$ and $q_y = u$ into \eqref{eqn:lemma-inverse-retraction}, we expand $v$ as
\begin{equation*}
v = \Inv{\R}_{q_x}(q_y) = (q_y - q_x) + \mathcal{O}(\Norm{q_y - q_x}^2) = 2u + \mathcal{O}(\Norm{u}^2).
\end{equation*}
Recall from \eqref{eqn:retraction-midpoint} that the retraction midpoint is defined as $\hat{q}_{\mathrm{mid}} = \R_{q_x}(\tfrac{1}{2} v)$.
Substituting $q_x=-u$ and $v=2u + \mathcal{O}(\Norm{u}^2)$, we expand $\hat{q}_{\mathrm{mid}}$ using \eqref{eqn:lemma-retraction}
\begin{equation*}
\hat{q}_{\mathrm{mid}}
= \R_{q_x}(\tfrac{1}{2} v)
= -u + \left( u + \mathcal{O}(\Norm{u}^2) \right) + \mathcal{O}(\Norm{u}^2)
= \mathcal{O}(\Norm{u}^2).
\end{equation*}
Since $\Norm{\delta} = \mathcal{O}(\Norm{u}^2)$, substituting $\Norm{u} = h/2$ gives $\Norm{\delta} = \mathcal{O}(h^2)$.\qed
\end{proof}

\begin{lemma}
\label{lemma:inverse-retractions-difference}
The difference of inverse retractions at $\hat{q}_{\mathrm{mid}}$ satisfies
\begin{equation*}
\Inv{\R}_{\hat{q}_{\mathrm{mid}}}\left( q_{y} \right) - \Inv{\R}_{\hat{q}_{\mathrm{mid}}}\left( q_{x} \right)
= 2u + \mathcal{O}(\Norm{u}^3).
\end{equation*}
\end{lemma}

\begin{proof}
Let $\R(z,\zeta)$ denote the coordinate representation of the retraction, that is, $\R(z,\zeta)$ is the coordinate of $\R_{\exp_q(z)}(\zeta)$ in the same chart.
Since $\R(0,0)=0$ and $D_\zeta \R(0,0)=I$, a second-order Taylor expansion yields
\begin{equation}
\label{eqn:lemma-local-retraction-expansion}
\R(z, \zeta) = z + \zeta + \mathcal{Q}(\zeta,\zeta) + \mathcal{B}(z,\zeta) + \mathcal{O}(\Norm{(z,\zeta)}^3),
\end{equation}
where $\mathcal{Q}$ is quadratic in $\zeta$ and $\mathcal{B}$ is bilinear in $(z,\zeta)$.
Because $\R_{\hat{q}_{\mathrm{mid}}}$ is a local diffeomorphism near $0\in \mathcal{T}_{\hat{q}_{\mathrm{mid}}}\M$, the inverse-retraction vectors
$w_y = \Inv{\R}_{\hat q_{\mathrm{mid}}}(q_y)$ and $w_x = \Inv{\R}_{\hat q_{\mathrm{mid}}}(q_x)$
are well-defined (for $\Norm{u}$ small) and satisfy
\begin{equation*}
\R(\delta,w_y)=u, \quad \R(\delta,w_x)=-u.
\end{equation*}
We solve for $w_y$ and $w_x$ given base point $z = \delta$ and targets $\pm u$.
Let $w_y = u + e_y$ and $w_x = -u + e_x$, where $e_y,e_x$ are correction terms.
Substituting into \eqref{eqn:lemma-local-retraction-expansion}
\begin{align*}
u &= \R(\delta, u + e_y) = \delta + (u + e_y) + \mathcal{Q}(u, u) + \mathcal{B}(\delta, u) + \mathcal{O}(\Norm{u}^3),\\
-u &= \R(\delta, -u + e_x) = \delta + (-u + e_x) + \mathcal{Q}(-u, -u) + \mathcal{B}(\delta, -u) + \mathcal{O}(\Norm{u}^3).
\end{align*}
Solving for the errors,
\begin{align*}
e_y &= - \delta - \mathcal{Q}(u, u) - \mathcal{B}(\delta, u) + \mathcal{O}(\Norm{u}^3),\\
e_x &= - \delta - \mathcal{Q}(-u, -u) - \mathcal{B}(\delta, -u) + \mathcal{O}(\Norm{u}^3).
\end{align*}
Using the symmetry $\mathcal{Q}(-u, -u) = \mathcal{Q}(u, u)$ and bilinearity $\mathcal{B}(\delta, -u) = -\mathcal{B}(\delta, u)$, subtracting the two expressions cancels out the even terms $-\delta$ and $-\mathcal{Q}$
\begin{equation*}
w_y - w_x = 2u + (e_y - e_x) = 2u + (-2 \mathcal{B}(\delta, u) + \mathcal{O}(\Norm{u}^3)).
\end{equation*}
Since $\mathcal{B}$ is bilinear, and by Lemma~\ref{lemma:retraction-midpoint}, $\Norm{\delta} = \mathcal{O}(\Norm{u}^2)$, it follows that $\mathcal{B}(\delta, u) = \mathcal{O}(\Norm{\delta}\Norm{u}) = \mathcal{O}(\Norm{u}^3)$.
Therefore, we have
\begin{equation*}
w_y - w_x
= \Inv{\R}_{\hat{q}_{\mathrm{mid}}}\left( q_{y} \right) - \Inv{\R}_{\hat{q}_{\mathrm{mid}}}\left( q_{x} \right)
= 2u + \mathcal{O}(\Norm{u}^3).
\tag*{\qed}
\end{equation*}
\end{proof}

\begin{proof}[Theorem~\ref{thm:accuracy}]
In normal coordinates, the metric tensor at $\hat{q}_{\mathrm{mid}}$ expands as
\begin{equation*}
G(\delta) = I + \mathcal{O}(\Norm{\delta}^2).
\end{equation*}
Substituting $\Norm{\delta} = \mathcal{O}(\Norm{u}^2)$ from Lemma~\ref{lemma:retraction-midpoint}, we get $G(\delta) = I + \mathcal{O}(\Norm{u}^2)$.
Let $\Delta w = w_y - w_x$.
Using Lemma~\ref{lemma:inverse-retractions-difference}, we write $\Delta w = 2u + r$, where $r = \mathcal{O}(\Norm{u}^3)$.
The midpoint retraction distance in \eqref{eqn:approx-midpoint-distance} is the norm of $\Delta w$ under $G(\delta)$, which relates to the Euclidean norm $G(0)$ by the metric expansion
\begin{align*}
\hat{d}_{\mathcal{M}}(q_x, q_y) &= \Norm{2u + r}_{G(\delta)}\\
&= \Norm{2u + r}_{G(0)} \left(1 + \mathcal{O}(\Norm{u}^2)\right)\\
&= \left(2\Norm{u} + \mathcal{O}(\Norm{u}^3)\right) \left(1 + \mathcal{O}(\Norm{u}^2)\right)\\
&= 2\Norm{u} + \mathcal{O}(\Norm{u}^3).
\end{align*}
Since $d_{\mathcal{M}}(q_x, q_y) = 2\Norm{u} = h$, we have
\begin{equation*}
\left| \hat{d}_{\mathcal{M}}(q_x, q_y) - d_{\mathcal{M}}(q_x, q_y) \right| = \mathcal{O}(h^3).
\tag*{\qed}
\end{equation*}
\end{proof}

\bibliographystyle{splncs04}
\bibliography{main}

\end{document}